\documentclass[10pt,oneside,twocolumn,a4paper]{article}
\usepackage{graphicx}
\usepackage[utf8]{inputenc}
\usepackage{mathptmx}
\usepackage{tabularx}
\usepackage{ragged2e}
\usepackage[singlelinecheck=false]{caption}
\usepackage[T1]{fontenc}
\usepackage[numbers]{natbib}
\usepackage{amsmath}
\RequirePackage[blocks]{authblk}
\usepackage{babel}
\usepackage{amsthm}
\usepackage{amssymb,amsfonts}
\usepackage{algorithm}
\usepackage{algpseudocode}
\usepackage{hyperref}
\usepackage{marvosym}
\usepackage{pifont}
\usepackage{tikz}
\usetikzlibrary{positioning}

\newtheorem{definition}{Definition}

\newtheorem{theorem}{Theorem}

\newtheorem{example}{Example}

\makeatletter
\let\ps@plain\ps@empty
\def\@xivpt{14bp}

\def\@sect#1#2#3#4#5#6[#7]#8{%
  \ifnum #2>\c@secnumdepth
    \let\@svsec\@empty
  \else
    \refstepcounter{#1}%
    \protected@edef\@svsec{%
      \ifnum #2<4
        \hb@xt@10mm{\csname the#1\endcsname}\relax
      \else
        \hb@xt@12mm{\csname the#1\endcsname}\relax
      \fi}%
  \fi
  \@tempskipa #5\relax
  \ifdim \@tempskipa>\z@
    \begingroup
      #6{%
        \@hangfrom{\hskip #3\relax\@svsec}%
          \interlinepenalty \@M #8\@@par}%
    \endgroup
    \csname #1mark\endcsname{#7}%
    \addcontentsline{toc}{#1}{%
      \ifnum #2>\c@secnumdepth \else
        \protect\numberline{\csname the#1\endcsname}%
      \fi
      #7}%
  \else
    \def\@svsechd{%
      #6{\hskip #3\relax
      \@svsec #8}%
      \csname #1mark\endcsname{#7}%
      \addcontentsline{toc}{#1}{%
        \ifnum #2>\c@secnumdepth \else
          \protect\numberline{\csname the#1\endcsname}%
        \fi
        #7}}%
  \fi
  \@xsect{#5}}
\renewcommand\LARGE{\@setfontsize\LARGE{16}{20}}
\def\abstract#1{\def\@abstract{#1}}
\def\abstractEn#1{\def\@abstractEn{#1}}
\def\titleEn#1{\def\@titleEn{#1}}
\def\@maketitle{%
  \newpage
  \null
  \let \footnote \thanks
    {\LARGE\bfseries\RaggedRight \@title \par}%
    {\LARGE\bfseries\RaggedRight \@titleEn \par}%
    \vskip 1\baselineskip%
    {\normalsize
      \@author\par}%
    \vskip \baselineskip%
    {\section*{Abstract}
      \@abstractEn}%
  \par
  \vskip 3\baselineskip}

\renewcommand\section{\@startsection {section}{1}{\z@}%
                                   {-3.5ex \@plus -1ex \@minus -.2ex}%
                                   {\baselineskip}%
                                   {\normalfont\Large\bfseries\RaggedRight}}
\renewcommand\subsection{\@startsection{subsection}{2}{\z@}%
                                     {\baselineskip}%
                                     {1ex}%
                                     {\normalfont\large\bfseries\RaggedRight}}
\renewcommand\subsubsection{\@startsection{subsubsection}{3}{\z@}%
                                     {1\baselineskip}%
                                     {3bp}%
                                     {\normalfont\normalsize\bfseries\RaggedRight}}
\renewcommand\paragraph{\@startsection{paragraph}{4}{\z@}%
                                    {1\baselineskip\@plus1ex \@minus.2ex}%
                                    {3bp}%
                                    {\normalfont\normalsize\RaggedRight}}
\renewcommand\subparagraph{\@startsection{subparagraph}{5}{\parindent}%
                                       {3.25ex \@plus1ex \@minus .2ex}%
                                       {-1em}%
                                      {\normalfont\normalsize\bfseries\RaggedRight}}
\parindent\p@
\makeatother
\DeclareCaptionLabelSeparator{enskip}{\enskip}
\usepackage{listings}

\newcommand\Mark[1]{\textsuperscript#1}

\title{Data-Driven Probabilistic Evaluation of Logic Properties with \\ PAC-Confidence on Mealy Machines}
\titleEn{}
\author{Swantje Plambeck\Mark{1}, Ali Salamati\Mark{2}, Eyke Hüllermeier\Mark{2}, Goerschwin Fey\Mark{1}}
\affil{\Mark{1}Hamburg University of Technology, \Mark{2}Ludwig-Maximilians-Universit\"{a}t M\"{u}nchen}
\affil{E-Mail: swantje.plambeck@tuhh.de, salamati@tum.de}

\abstractEn{
	Cyber-Physical Systems (CPS) are complex systems that require powerful models for tasks like verification, diagnosis, or debugging.
	Often, suitable models are not available and manual extraction is difficult.
	Data-driven approaches then provide a solution to, e.g., diagnosis tasks and verification problems based on data collected from the system.
	In this paper, we consider CPS with a discrete abstraction in the form of a Mealy machine.
	We propose a data-driven approach to determine the safety probability of the system on a finite horizon of $n$ time steps.
	The approach is based on the Probably Approximately Correct (PAC) learning paradigm.
	Thus, we elaborate a connection between discrete logic and probabilistic reachability analysis of systems, especially providing an additional confidence on the determined probability. 
	The learning process follows an active learning paradigm, where new learning data is sampled in a guided way after an initial learning set is collected.
	We validate the approach with a case study on an automated lane-keeping system.
}

\newcommand{\pr}{\mathbb{P}}

\begin{document}

\maketitle

\section{Introduction}
\label{sec:intro}

Cyber-Physical Systems (CPS) are systems which combine continuous physical processes with discrete behavioral modes.
CPS are ubiquitous in a wide range of applications, spanning from self-driving cars, power grids, and traffic networks to integrated medical devices.
The assurance of safety for such intricate systems is of high significance.

Whenever a model is available, reachability analysis and model checking can be applied to verify system properties~\cite{HR:2004}.
Such formal methods have extensions also for hybrid systems~\cite{tabuada09}, discrete dynamical systems~\cite{belta2017formal}, and probabilistic cases~\cite{kwiatkowska2009prism}.
To use formal methods also on CPS with continuous dynamics, abstraction-based methods as in \cite{zhang2010safety} are developed to verify the safety on a discrete abstraction of a hybrid system. \cite{zhang2010safety} use a method inspired by abstraction refinement to find an abstraction of a probabilistic hybrid automaton to verify safety properties.

In this paper, we focus on the discrete domain of systems modeled by Mealy machines.
Finite State Machine (FSM) models, especially Mealy machines are widely used to model the discrete behavior of CPS like communication protocols \cite{AHKOV2012}.
Instead of formal verification, we determine a probabilistic reachability, i.e., the probability of reaching a set of safe states.

Unfortunately, a valid discrete model in form of a Mealy machine is often not available. Manually extracting a precise model that fully represents a system is difficult or even impossible.
To overcome this problem, there has been a significant trend towards developing data-driven approaches to learn discrete models of complex systems~\cite{IHS2014,kazemi2024data}.
The same holds for the domain of safety analysis where data-driven approaches are used to either verify the safety of safety-critical systems or to synthesize safe controllers.

With this paper, we combine the ideas of discrete logic and probabilistic analysis. Precisely, we draw from the following concepts:
\begin{itemize}
	\item FSMs and Mealy machines as models for the discrete behavior of CPS,
	\item Probabilistic reachability analysis to determine the probability of reaching a set of safe states within a finite horizon of $n$ time steps,
	\item Probably Approximately Correct (PAC) learning algorithms to learn safety properties in a data-driven manner,
	\item Probabilistic safety guarantees to provide a safety guarantee with a confidence level.
\end{itemize}
With this combination, we provide a data-driven approach to determine the safety of systems modeled by Mealy machines with a probabilistic confidence.

For our approach, we assume a real-world system, that has an FSM abstraction, specifically a Mealy machine representation.
For such a representation, we derive a probability measure for reaching a set of safe states on a finite horizon of $n$ time steps.
We base on the PAC learning algorithms described by Valiant~\cite{Valiant1984} to learn the set of safe paths of the system.
This learning step is a data-driven method collecting observations on the system and follows an active learning paradigm.
Thus, during learning new learning data is sampled in a guided way.
The guarantees of PAC learning provide a probabilistic confidence on the accuracy of the learned set of safe paths.

The PAC formulation has two probabilistic levels. First, a probabilistic guarantee on the accuracy of the learned set and second a confidence of this guarantee based on the amount of data used for learning.
We add a third probabilistic level as a result of the original reachability problem. 
This probability serves, e.g., as a safety guarantee for the real-world system.

An additional contribution of this paper is the implementation of the proposed method including the PAC learning algorithms stated by Valiant~\cite{Valiant1984}.
Our implementation is publicly available.\footnote{\url{https://github.com/TUHH-IES/golddnf4safety}}
In an evaluation, we apply the method to a real-world-inspired system, the Automated Lane-Keeping System (ALKS).
The evaluation also compares to a purely stochastic method, which supports the capabilities of our approach in determining the safety of a system.

In the next section, we provide related work on data-driven model learning and safety verification.
In Section~\ref{sec:problemstatement}, the concept of a Mealy machines is introduced through a coffee machine example. Additionally, the problem is outlined, and the concept of reachability is formally defined. In Section~\ref{sec:modelbased}, we present the main result, which allows us to apply a probability measure to the probabilistic reachability of a Mealy machine with an unknown model.
The proposed method is evaluated in Section~\ref{sec:casestudies} through a practical case study. Section~\ref{sec:conclusion} concludes our paper.

\section{Related Work}

While our approach focuses on discrete system abstractions, there are several related works for the complete range of fully continuous systems over hybrid systems and discrete abstractions.
For safety verification in the continuous domain, the authors in \cite{prajna2004safety} and \cite{yang2020efficient} use a so-called barrier certificate to verify the safety of nonlinear and hybrid systems. A barrier certificate is a function that can separate safe and unsafe regions for a dynamical system.

The use of data to construct abstractions and check properties of dynamical systems has been studied in \cite{kazemi2024data} and \cite{majumdar2023neural}.
In \cite{kazemi2024data}, a data-driven method was introduced to learn abstractions of a dynamical system, along with formal confidence bounds to ensure the correctness of the learned abstractions.
Building on this, \cite{majumdar2023neural} proposed a compression scheme using neural network representations to alleviate the memory bottleneck of the abstraction-based method in \cite{kazemi2024data}, while preserving the formal guarantees.
A data-driven and model-based formal verification approach for partially unknown Linear Time-Invariant (LTI) systems is presented in \cite{haesaert2017data}.
In these works, the authors propose a method based on Bayesian inference and reachability analysis to provide confidence that a physical system affected by noisy measurements satisfies a given bounded-time temporal logic specification.
In \cite{polgreen2017automated}, a method based on Bayesian inference and model checking is developed for Markov decision processes. The recent results in \cite{asalamati2020} extend those of \cite{haesaert2017data} and \cite{polgreen2017automated} to the verification of stochastic LTI systems under specifications expressed as signal temporal logic formulae.
A fully data-driven approach for safety verification is proposed in \cite{salamati2024data} to ensure the safety of stochastic dynamical systems when the model is unknown. As the aforementioned works focus on continuous systems, our work focuses on the discrete domain. Additionally, we do not use a conventional stochastic approach, but instead our data-driven probabilistic statement bases on the PAC learning guarantees by Valiant~\cite{Valiant1984} and the analysis of logic expressions.

There exist many works on automata learning, but they focus on the learning of the model itself, not on safety verification. Seminal works are \cite{Valiant1984} and \cite{Angluin1987} which prove general learnability properties and outline learning algorithms. The open source framework LearnLib incorporates several paradigms and algorithms for automata learning \cite{IHS2015} and there are further recent algorithms for automata learning \cite{IHS2014,AP2018,GSPO2015}. One reason why learned automaton models are not used for safety verification is that automaton models are not guaranteed to be correct in practical learning scenarios. Nevertheless, the theory on Probably Approximately Correct (PAC) learning \cite{Valiant1984,kearns1994} can be used to provide a probabilistic guarantee on the accuracy of the learned model.
Other works utilizing the PAC learning framework for safety verification are \cite{chen2016pac} and \cite{prandini2023datadriven}. In \cite{chen2016pac}, the authors use PAC learning for probabilistic software verification. A PAC learning algorithm for automata learning is used to learn an automaton model of the feasible paths in the program. This model is then used to determine with PAC confidence that there exist no feasible error path or counterexample error paths are identified.
In \cite{prandini2023datadriven}, the authors use PAC learning for a probabilistic reachable set of dynamical systems. Two methods to estimate the reachable set are suggested. The first method uses a convex scenario optimization. The second approach uses empirical risk minimization. In contrast to this work, we focus on discrete reachability and, thus, use classical PAC formulations.

\section{Preliminaries}
\label{sec:problemstatement}

In this section, we introduce preliminary concepts such as PAC learning and Mealy machines. Additionally, we formalize the reachability problem.

\subsection{PAC Learning of Boolean Expressions}
\label{sec:pac-valiant}

The Probably Approximately Correct (PAC) learning approach is introduced by Valiant in \cite{Valiant1984} for learning logical programs. Programs are PAC-learnable if 
\begin{enumerate}
	\item There exists a learning algorithm that has polynomial complexity in an adjustable parameter $h$ and the number $t$ of variables.
	\item The learning algorithm deduces at least with probability $(1-h^{-1})$ a program that never outputs one when it should not and outputs zero when it should not at most with probability $h^{-1}$.
\end{enumerate}

Valiant proves that boolean expressions are PAC-learnable if the cardinality of the learning data set $L$ fulfills the following condition
\begin{equation}
	L \geq 2 h (d + \log_e(h)),
	\label{eq:pac-valiant}
\end{equation}
where $d$ is the number of positive results of the program.

Valiant \cite{Valiant1984} also specifies corresponding learning algorithms for learning Conjunctive Normal Form (CNF) and Disjunctive Normal Form (DNF) expressions.

\subsection{Automata \& Mealy Machines}
The discrete system behavior of a CPS can be modeled by a Mealy machine as it represents the relation between inputs and outputs of the system over a set of discrete states.

\begin{figure}[h]
	\centering
	\includegraphics[width=1\linewidth]{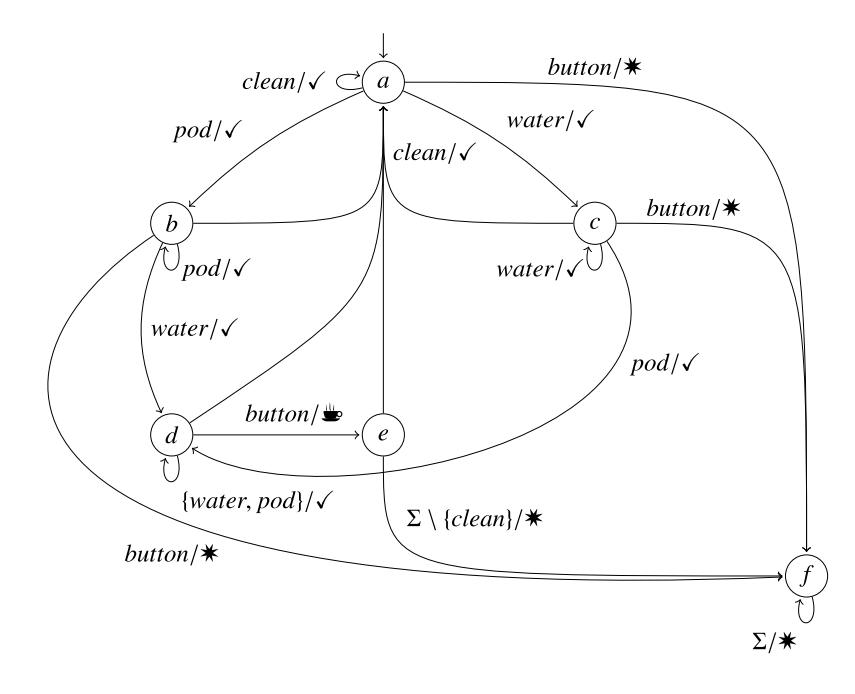}
	\caption{Exemplary Automaton model \cite{SHM2011}}
	\label{fig:coffee}
\end{figure}

\begin{definition}
	A Mealy machine $A$ is a tuple $(S,I,O, \alpha, \beta, q_0)$, where
	$S$ is the set of states,
	$I$ is the input alphabet,
	$O$ is the output alphabet,
	$\alpha : I \times S \rightarrow S$ is the transition function,
	$\beta : I \times S \rightarrow O$ is the output function, and
	$q_0 \in Q$ is the initial state. 
	\label{def:deterministicc}
\end{definition}

\begin{example}
	Fig.~\ref{fig:coffee} shows a coffee machine modelled as a Mealy machine. The figure shows a Mealy machine with six states and transitions between the states which are labeled by the input and output symbols.
	The automaton model in Fig.~\ref{fig:coffee} can be expressed as the following Mealy machine. 
	\begin{itemize}
		\item $S = \{a,b,c,d,e,f\}$
		\item $I = \{\text{clean}, \text{button}, \text{water}, \text{pod}\}$
		\item $O = \{$\checkmark, \Coffeecup, \ding{88}$\}$ 
		\item $q_0=a$
	\end{itemize}
	\label{ex:coffee}
\end{example}

\subsection{Reachability Problem}

The forward reachability problem is a fundamental problem in formal verification. Here, we consider reachability on a finite horizon of $n$ time steps. Given a transition function $\alpha$ of a system, we derive a tracing function $\Phi : I^n \rightarrow S$:
\begin{align}
	\Phi_n(i_1,...,i_n) = \alpha(i_n,...,\alpha(i_2,\alpha(i_1,q_0))).
	\label{eq:tracing}
\end{align}
The tracing function $\Phi_n$ provides the state of the system after $n$ time steps for a given input sequence $i_1,...,i_n$ starting from the initial state $q_0$.

With this tracing function, the reachable set of states $R_n$ on a finite horizon of length $n$ is defined as follows:
\begin{align}
	R_n = \{\Phi_n(i_1,...,i_n) | i_1,...,i_n \in I\}.
	\label{eq:reachable}
\end{align}
Thus, the reachable $R_n$ contains all states that are reachable on a finite horizon of $n$ time steps.

\section{Data-Driven Probabilistic Reachability Analysis of Mealy Machines}
\label{sec:modelbased}

In this section, we introduce a model-based approach to analyze the probability of reaching a set of states on a finite horizon of a system.
The system is observed with a discrete abstraction.
We assume that the system has an implicit Mealy machine representation $A$.
Further, there is a set of states $\mathcal{X} \subseteq S$ for which we determine the probability of reaching this set on a finite horizon of $n$ time steps, i.e.
\begin{equation}
	\pr(x \in \mathcal{X} | x \in \Phi_n(i_1,...,i_n)).
	\label{eq:reachability}
\end{equation}

\subsection{Probabilistic Reachability Formulation}

As the implicit Mealy machine representation is unknown, also the set $\mathcal{X}$ is usually not directly given.
Instead, we assume that whether a state in $\mathcal{X}$ is reached is \emph{observable}, i.e., after execution of an input sequence, it is clear whether a state in $\mathcal{X}$ has been reached or not.
Even though the approach is agnostic to whatever the meaning of $\mathcal{X}$ is, we assume for simplicity that $\mathcal{X}$ is the set of safe states of the system.
With this, it is also likely that we know whether an observation is safe or not.

\begin{example}
	For the coffee machine example from Example~\ref{ex:coffee}, we consider the set of safe states $\mathcal{X} = S \setminus \{f\}$.
	Reaching the unsafe state $f$ is observed through the output \emph{error} (represented by the $\ast$ in Fig.~\ref{fig:coffee}).
\end{example}

Our goal is to determine the probability of the system to be safe on a finite horizon of $n$ time steps.
For the unknown Mealy machine abstraction, we do this using a data-driven approach.
We use a PAC-based learning approach which provides a confidence for the safety probability.
The idea is similar to an explicit-state model checking~\cite{holzmann2018explicit}, i.e., we try to enumerate the safe paths of the system.
The enumeration is done with a set $g$:
\begin{equation}
	g = \{m_1,...,m_{|g|}\}, \quad m_k = \{\langle\tau_1,i_1\rangle,...,\langle\tau_{l_k},i_{l_k}\rangle\},
	\label{eq:maps}
\end{equation}
where $l_k$ is the length of the $k$-th map $m_k$ in $g$.
The set $g$ consists of path-generalizations $m_k$.
Every path generalization is represented by a map between time steps $\tau \in [1..n]$ to input $i \in I$.
Monomials encode paths on the system by encoding input sequences starting in the initial state of the system.
\begin{align*}
	& m = \{\langle\tau_1,i_1\rangle,...,\langle\tau_{l_i},i_{l_i}\rangle\} \mapsto\\
	& \left\{ [\tilde{i}_1, ... \tilde{i}_n] \text{ with } \tilde{i}_{j} \in
	\begin{cases}
		\{m[\tau_j]\} & \text{if } \tau_j \in \{\tau_k | k\in [1..l_i]\} \\
		I & \text{otherwise}
	\end{cases}\right\}
\end{align*}
The input sequences are formed from the maps $m$ by taking the input at each time step for all tuples in $m$.
All time steps that are not in $m$ are considered as \emph{don't care} values.
Thus, for every \emph{don't care} value, all inputs are possible which results in a set of input sequences for this map.

\begin{example}
	\label{ex:coffee-dnf}
	For the coffee machine from Fig.~\ref{fig:coffee} and $n=2$ time steps, the input sequences for the safe paths are
	\begin{align*}
		\{[\text{clean},\text{clean}],[\text{clean},\text{water}],[\text{clean},\text{pod}], \\
		[\text{water},\text{pod}], [\text{water}, \text{water}],[\text{water},\text{clean}],
		\\ [\text{pod},\text{water}],[\text{pod},\text{pod}], [\text{pod},\text{clean}]\}.
	\end{align*}
	The input sequences result in the set
	\begin{align*}
		g = \{\{\langle 1, \text{clean}\rangle, \langle 2, \text{clean}\rangle\}, \\
		\{\langle 1, \text{clean}\rangle, \langle 2, \text{water}\rangle\}, \\
		\{\langle 1, \text{clean}\rangle, \langle 2, \text{pod}\rangle\}, \\
		\{\langle 1, \text{water}\rangle, \langle 2, \text{pod}\rangle\}, \\
		\{\langle 1, \text{water}\rangle, \langle 2, \text{water}\rangle\}, \\
		\{\langle 1, \text{water}\rangle, \langle 2, \text{clean}\rangle\}, \\
		\{\langle 1, \text{pod}\rangle, \langle 2, \text{water}\rangle\}, \\
		\{\langle 1, \text{pod}\rangle, \langle 2, \text{pod}\rangle\}, \\
		\{\langle 1, \text{pod}\rangle, \langle 2, \text{clean}\rangle\} \}.
	\end{align*}
\end{example}

Having a set $g$ at hand, we can reversely determine the number of safe paths $x_S$ of length $n$ as follows:
\begin{align}
	x_S = \sum_{m_k \in g} |I|^{n-l_k},
	\label{eq:paths}
\end{align}
where $l_k$ is again the length of the $k$-th map $m_k$ in $g$.

Knowing the number of safe paths and assuming that all sequences are observed equally likely, we finally derive the safety probability of the system on a finite horizon of $n$ time steps as follows:
\begin{align}
	\pr(\text{Safety}) = \frac{x_S}{|I|^n}.
	\label{eq:safety-prob}
\end{align}
This gives the ratio of safe paths, found from $g$ as $x_S$, to the total number of possible paths $|I|^n$.

\subsection{Data-Driven Identification}

The Mealy machine representation of the system is unknown, but we are able to simulate the system and to collect data from the abstract Mealy machine and use this data to learn the safety probability of the system.
We construct $g$ in a data-driven manner by collecting safe observations on the system. Then, we generalize observations to maps $m_k, k\in [1, \dots, |g|]$.
This generalization determines whether the input at any time step is relevant for the safety of the system by checking the safety of the system for all possible inputs at this time step.
Non-relevant inputs are removed from $m_k$.
The learning process starts from an initial learning set consisting of $L$ input sequences which lead to safe paths in the system.
These paths form candidate maps.
The generalization procedure then implements an active learning scenario where additional queries are made to the system to determine whether safe paths are generalizable.
The queries are assumed to answer correctly, i.e., a query always provide information on whether a sample is safe or unsafe.
This guided learning process collects samples that provide relevant information for the learning process.

Next, we introduce the main theorem that specifies the safety probability of an unknown Mealy machine using an initial learning set of size $L$.

\begin{theorem}
	Given a deterministic Mealy machine as defined in Definition~\ref{def:deterministicc} and a set of size $L$.
	Assume  $M = |I|^n$, where $|I|$ is the number of inputs of the Mealy machine, and $n$ is number of time steps. One has
	\begin{align}
		\pr(\text{Safety}) = \frac{x_S}{M},
		\label{eq:safety}
	\end{align}
	where $x_S$ is the number of safe paths of length $n$ in the system according to Equation~\ref{eq:paths}.
	Equation~\ref{eq:safety} holds with a confidence of at least $1-h^{-1}$ if the cardinality $L$ of the initial learning data fulfills the following condition
	\begin{equation}
		L \geq 2 h (x_S + \log_e(h)),
		\label{eq:pac}
	\end{equation}
	where $h$ is a chosen parameter.
	\label{theorem:main}
\end{theorem} 

\begin{proof}
	This learning procedure is equivalent to the PAC learning of a Boolean expressions as described in Section~\ref{sec:pac-valiant}.
	Thus, PAC learning guarantees apply as follows: the learned set $g$ with probability $1-h^{-1}$ never considers a non-safe path as safe and at most with probability $h^{-1}$ considers a safe path as non-safe, if we initially use a learning set of size 
	\begin{equation*}
		L \geq 2 h (S + \log_e(h)).
	\end{equation*}
	According to the PAC learning formulation, $S$ is the maximum number of positive results which is the number $x_S$ of safe paths in the system.
	This number can be determined from the learned set $g$ as given in Equation~\ref{eq:paths}.
\end{proof}

Theorem~\ref{theorem:main} is a two staged probabilistic statement, where the first stage in Equation~\ref{eq:safety} provides the safety probability that results from the inspection of the learned set $g$.
The second stage in Equation~\ref{eq:pac} provides a probabilistic guarantee and confidence on the accuracy of the learned set in determining safe paths using the PAC formulation.
Note that the number of samples $L$ implicitly depends on the length $n$ of the input sequences as a larger number of time steps implies possibly more maps, i.e., larger $x_S$. 

\begin{example}
	For the coffee machine from Fig.~\ref{fig:coffee}, we start with $L=1000$ samples and $N=5$. The learned set $g$ has $x_S=272$ as all maps $m_k$ have a length of 5. This set is correct with a confidence of $1-h^{-1} = 0.45$, where $h=1.83$, i.e., 
	\begin{equation*}
		1000 \geq 997.7 = 2 \cdot 1.83 (272 + \log_e(1.83)).
	\end{equation*}
	The safety probability is $\frac{272}{|I|^N} = \frac{272}{4^5} = 0.27$.
\end{example}


\subsection{Implementation}

We implement the algorithm for learning a DNF from data as proposed in \cite{Valiant1984} and adapt it to our scenario. The algorithm is implemented in Java to interact with the automata learning framework LearnLib~\cite{IHS2015}. The implementation is publicly available on GitHub~\cite{git-repo}.

The main process for safety evaluation contains the following steps:
\begin{enumerate}
	\item Define the confidence level $1-h^{-1}$.
	\item Determine the number of samples $L$ according to Equation~\ref{eq:pac}.
	\item Learn the set $g$ of the safety property using $L$ samples according to \cite{Valiant1984} and calculate the number of safe paths $x_S$ according to Equation~\ref{eq:paths}.
	\item Compute $M = |I|^N$.
	\item Calculate the safety property $\pr(\text{safety})=\frac{x_S}{M}$.
\end{enumerate}
Input to the algorithm is the parameter $h$ which determines the desired confidence level.
Based on the confidence level, the amount of initial learning data $L$ is derived.
The algorithm learns a set $g$ for the safety property of the Mealy machine model. The safety probability is calculated based on the size of the set $g$. 
In the case that instead of the confidence level only an initial learning set of size $L$ is known, an alternative strategy is to first learn the set $g$ based on this initial learning set. Then, the number of safe paths $x_S$ can be determined according to Equation~\ref{eq:paths}. Solving Equation~\ref{eq:pac} for $h$ finds the confidence level for the initial learning set.
If a desired confidence level is not reached, the algorithm can be rerun with a larger sample size.

Algorithm~\ref{algo:learn} implements the algorithm for identification of DNF formulae proposed by Valiant in \cite{Valiant1984}.

\begin{algorithm}
	\begin{algorithmic}[1]
		\Function{learnSafeSet}{Automaton A, Int L}
		\State const n = length of history
		\State g = False
		\For{$i \gets 1 \text{ to } L$}
		\State $v \gets$ \Call{getExample}{A, $n$}
		\If{$v \nRightarrow g$}
		\For{$i = 1 \text{ to } t$}
		\If{$\tau_i$ is determined by $v$}
		\State $\tilde{v} \gets v(\tau_i=\ast)$
		\If{\Call{Oracle}{A, $\tilde{v}$} = True}
		\State $v \gets \tilde{v}$
		\EndIf
		\EndIf
		\EndFor
		\State $m \gets$ time steps determined in $v$
		\State $g \gets g \sup \{m\}$
		\EndIf
		\EndFor
		\State \Return{$g$}
		\EndFunction
	\end{algorithmic}
	\caption{Algorithm to Learn a DNF Formula}
	\label{algo:learn}
\end{algorithm}

The function \emph{LearnSafeSet} uses the two helper functions \emph{getExample} and \emph{Oracle} which are defined in Algorithm~\ref{algo:sample} and Algorithm~\ref{algo:oracle}, respectively. The function \emph{getExample} samples a random, safe input sequence from the automaton and converts it to a candidate map. The function \emph{Oracle} creates all possible sequences that can be derived from the map $v$. Afterwards, it queries the automaton for each of these sequences and returns true if a safe state is reached, otherwise false is returned.

\begin{algorithm}
	\begin{algorithmic}[1]
		\Procedure{getExample}{Automaton A, Int $n$}
		\While{True}
		\State $in \gets$ A.randomInput($n$)
		\State $v$ = convertToMonomial($in$)
		\If{A.apply($in$) is \textit{safe}}
		\State \Return{$v$}
		\EndIf
		\EndWhile
		\EndProcedure
	\end{algorithmic}
	\caption{Algorithm to Sample from a System}
	\label{algo:sample}
\end{algorithm}

\begin{algorithm}
	\begin{algorithmic}[1]
		\Function{Oracle}{Automaton A, Monomial $v$}
		\State $I_s$ = convertToInputSeqs($v$) \Comment{All sequences that are possible with the given assignments}
		\For{$in$ in $I_s$}
		\If{A.apply($in$) is \textit{safe}}
		\State \Return{True}
		\EndIf
		\EndFor
		\State \Return{False}
		\EndFunction
	\end{algorithmic}
	\caption{Algorithm to Query the Oracle}
	\label{algo:oracle}
\end{algorithm}

The learning algorithm scales linearly with the number of samples $L$ and the number of possible variables $t$ in $g$ \cite{Valiant1984}.
The number of variables $t$ is determined by the number of inputs and time steps, while $L$ depends on the number of safe paths according to Equation~\ref{eq:pac-valiant}.
The size of $g$ is determined by the complexity of the system and is bounded by the number of time steps and the input alphabet size.
In the worst case, complexity grows exponentially with the number of time steps and the input alphabet size.
Anyways, benign systems, e.g., with unsafe deadlocks, will have a much smaller size of $g$.
Furthermore, the Mealy machine usually represents an abstraction of an original, complex systems, thus, inherently reducing the complexity of the problem.

\section{Case Studies}
\label{sec:casestudies}

In this section, we present two case studies to validate the proposed methodology. In particular, we study an Automated Lane-Keeping System (ALKS).

\subsection{Automated Lane-Keeping System (ALKS)}

We consider a car with a manual steering system. We discuss two versions, one with and one without a supporting ALKS. The ALKS is a system that automatically steers the car to keep it in the lane. If the ALKS is used, the system supports the manual steering if the car is about to leave the lane.

\begin{figure}
	\centering
	\begin{tikzpicture}
		\node[circle, draw] at (0,0) (C) {C};
		\node[circle, draw] at (-2,2) (L) {L};
		\node[circle, draw] at (2,2) (R) {R};
		\node[circle, draw] at (0,4) (A) {A};
		
		\draw[->] (C) to [bend left] node[below,xshift=-0.9cm] {l / ok} (L);
		\draw[->] (L) to [bend left] node[above,xshift=0.5cm] {r / ok} (C);
		\draw[->] (C) to [bend left] (R);
		\draw[->] (R) to [bend left] node[below,xshift=0.9cm] {l / ok} (C);
		\draw[->] (L) to [bend left] node[above,xshift=-0.9cm] {l / alarm} (A);
		\draw[->] (R) to [bend right] node[above,xshift=0.9cm] {r / alarm} (A);
		
		\draw [->] (L) edge[loop left] node{s / ok} (L);
		\draw [->] (C) edge[loop below] node{s / ok} (C);
		\draw [->] (R) edge[loop right] node{s / ok} (R);
		
		\draw [->,dashed,blue] (A) edge[loop above] node{s,l,r / alarm} (A);
		\draw [->, dotted, red, xshift=4cm] plot [smooth, tension=3] coordinates {(A.east) (0.5,2) (C.east)} node[below,pos=0.4]{s,l,r / ok};
		
		\draw [->] (-1,-1) -- (C);
		
	\end{tikzpicture}
	\caption{Mealy machine model of a steering system which uses an ALKS (dashed-blue) and which does not use an ALKS (dotted-red)}
	\label{fig:alks}
\end{figure}
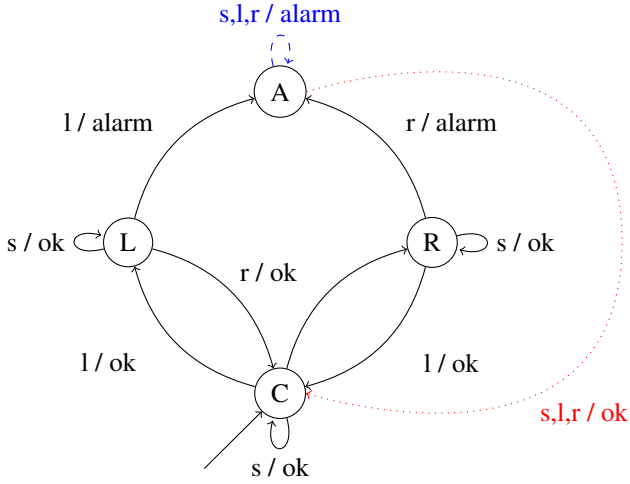

The Mealy machine representation of the system is shown in Fig.~\ref{fig:alks}. The states are $C$ for the car in the lane, $L$ for the car on the left lane boundary, $R$ for the car on the right lane boundary, and $A$ for the alarm state, i.e., out of lane boundaries. The inputs are $l$ for steering in the left direction, $r$ for steering in the right direction, and $s$ for the car moving straight. The outputs are $ok$ for the car in the lane and $alarm$ for the car out of the lane. The system with and without the ALKS differentiate in the transitions leaving state $A$. The system without ALKS stays in the alarm state while the system with ALKS counteracts the manual steering and brings the car back in the center of the lane, i.e., to state $C$.

In the following, we apply our method to estimate the safety probability from observed data with a confidence level determined by the size $L$ of the initial learning set.

\subsection{Results}

\begin{table}[]
	\centering
	\caption{Numerical Results of the Safety Analysis}
	\label{tab:results}
	\setlength{\tabcolsep}{0.4em}
	\begin{tabular}{|l|c|c|c|r|r|r|}
		\hline
		Example & N & d & L & $1-h^{-1}$ & $\pr_V$ & $\pr_L$ \\
		\hline
		wto ALKS & 3 & 17 & 1000 & 0.96 & 0.63 & 0.634\\
		wto ALKS & 4 & 41 & 1000 & 0.91 & 0.51 & 0.507\\
		wto ALKS & 5 & 99 & 1000 & 0.80 & 0.41 & 0.42\\
		wto ALKS & 10 & 952 & 1000 & 0.00 & 0.02 & 0.12\\
		ALKS & 3 & 23 & 1000 & 0.95 & 0.85 & 0.87\\
		ALKS & 4 & 71 & 1000 & 0.85 & 0.88 & 0.88\\
		ALKS & 5 & 207 & 1000 & 0.58 & 0.85 & 0.86\\
		ALKS & 10 & 988 & 1000 & $\approx$ 0.00 & 0.02 & 0.87\\
		\hline
	\end{tabular}
\end{table}

Table~\ref{tab:results} shows results for different time horizons $N$ and the ALKS with and without automated steering.
The number of samples used for learning is chosen to $L=1000$ for all examples. The column $1-h^{-1}$ gives the confidence in the found safety level $\pr_V$.
In addition to the safety level $\pr_V$, the safety level $\pr_L$ is given, which is a stochastic estimate of the safety level based on the evaluation of $L=1000$ random sequences of length $N$. The safety level $\pr_L$ is calculated as the ratio of safe sequences to the total number of sequences. The safety level $\pr_L$ is a stochastic estimate of the safety level $\pr_V$ and validates the results of the learning algorithm.

From the presented results, we observe that the confidence level falls as $N$ increases.
The reason for this is that for a larger time horizon more samples are needed to cover a larger range of the possible system behavior.
As we keep $L$ constant, the confidence falls with $N$.
Furthermore, we observe that the safety level $\pr_V$ decreases with increasing $N$ for the system without (wto) ALKS.
This is due to the fact that the Mealy machine representation shown above has a dead-end state in the unsafe state $A$.
Thus, the longer the time horizon, the higher the probability to end up in the dead-end state.
For the system with ALKS, the safety level is higher and stays constant for increasing $N$.
In this scenario, the system can recover from the unsafe state $A$ and return to the safe states.
Still, for large $N$, we have a low safety level, as the amount of samples in this case is not sufficient to estimate the number of safe paths.
This is adequately reflected in the confidence level $1-h^{-1}$ which is close to 0.
Still, the comparison of the safety level $\pr_V$ and $\pr_L$ shows that the learning algorithm struggles to estimate the safety level for large $N$.

An advantage of the approach in contrast to other formal methods is that it can be applied to systems where the model is unknown as it is data-driven, but the oracle needs access to the real system to collect data.
Even though the approach struggles in its scalability, the advantages in contrast to other stochastic methods are that it provides a PAC confidence level.
Furthermore, the active learning paradigm implemented in the oracle function allows for a more efficient, guided data collection and learning process.

\section{Conclusion}
\label{sec:conclusion}
The discrete behavior of many real-world applications such as autonomous systems can be modelled by Mealy machines. In this work, we develop a fully data-driven learning approach to assess the reachability problems, specifically focusing on reachability of safe regions of systems modeled by Mealy machines.
The approach combines several concepts from discrete logic, formal methods, stochastic analysis and safety verification from a new perspective.
We formulate the safety property within the context of a Mealy machine through a set of generalized paths encoding the safe paths of the system. Since the model is unknown, the safety behavior is learned via a PAC-learning approach from samples and active queries on the system.
The safety probability is provided with a PAC confidence.
We validate our methodology through practical case studies, demonstrating its efficacy in real-world scenarios.
Even though the approach struggles with scalability, the active learning algorithm implements efficient guided data collection and the resulting safety property provides a PAC confidence level.

\section*{Acknowledgment}
This work was supported by the German Ministry of Education and Research (BMBF) within the project AGenC under grant number 01IS22047A.
Furthermore, the authors would like to thank Maximilian Schmidt for the valuable feedback and Sadegh Soudjani for initial fruitful discussions. 

\bibliographystyle{ieeetran}
\bibliography{references1}

\begin{thebibliography}{10}
\providecommand{\url}[1]{#1}
\csname url@samestyle\endcsname
\providecommand{\newblock}{\relax}
\providecommand{\bibinfo}[2]{#2}
\providecommand{\BIBentrySTDinterwordspacing}{\spaceskip=0pt\relax}
\providecommand{\BIBentryALTinterwordstretchfactor}{4}
\providecommand{\BIBentryALTinterwordspacing}{\spaceskip=\fontdimen2\font plus
\BIBentryALTinterwordstretchfactor\fontdimen3\font minus \fontdimen4\font\relax}
\providecommand{\BIBforeignlanguage}[2]{{%
\expandafter\ifx\csname l@#1\endcsname\relax
\typeout{** WARNING: IEEEtran.bst: No hyphenation pattern has been}%
\typeout{** loaded for the language `#1'. Using the pattern for}%
\typeout{** the default language instead.}%
\else
\language=\csname l@#1\endcsname
\fi
#2}}
\providecommand{\BIBdecl}{\relax}
\BIBdecl

\bibitem{HR:2004}
M.~Huth and M.~Ryan, \emph{Logic in computer science : modelling and reasoning about systems}, 2nd~ed.\hskip 1em plus 0.5em minus 0.4em\relax Cambridge University Press, 2004.

\bibitem{tabuada09}
\BIBentryALTinterwordspacing
P.~Tabuada, \emph{Verification and Control of Hybrid Systems: A Symbolic Approach}.\hskip 1em plus 0.5em minus 0.4em\relax Springer, 2009. [Online]. Available: \url{http://books.google.nl/books?id=1ExhrqtzIYwC}
\BIBentrySTDinterwordspacing

\bibitem{belta2017formal}
C.~Belta, B.~Yordanov, and E.~A. Gol, \emph{Formal methods for discrete-time dynamical systems}.\hskip 1em plus 0.5em minus 0.4em\relax Springer, 2017, vol.~15.

\bibitem{kwiatkowska2009prism}
M.~Kwiatkowska, G.~Norman, and D.~Parker, ``Prism: Probabilistic model checking for performance and reliability analysis,'' \emph{ACM SIGMETRICS Performance Evaluation Review}, vol.~36, no.~4, pp. 40--45, 2009.

\bibitem{zhang2010safety}
L.~Zhang, Z.~She, S.~Ratschan, H.~Hermanns, and E.~M. Hahn, ``Safety verification for probabilistic hybrid systems,'' in \emph{International Conference on Computer Aided Verification}.\hskip 1em plus 0.5em minus 0.4em\relax Springer, 2010, pp. 196--211.

\bibitem{AHKOV2012}
F.~Aarts, F.~Heidarian, H.~Kuppens, P.~Olsen, and F.~Vaandrager, ``Automata learning through counterexample guided abstraction refinement,'' in \emph{FM 2012: Formal Methods}, D.~Giannakopoulou and D.~M{\'e}ry, Eds.\hskip 1em plus 0.5em minus 0.4em\relax Berlin, Heidelberg: Springer Berlin Heidelberg, 2012, pp. 10--27.

\bibitem{IHS2014}
M.~Isberner, F.~Howar, and B.~Steffen, ``The ttt algorithm: A redundancy-free approach to active automata learning,'' in \emph{Runtime Verification}, B.~Bonakdarpour and S.~A. Smolka, Eds.\hskip 1em plus 0.5em minus 0.4em\relax Cham: Springer International Publishing, 2014, pp. 307--322.

\bibitem{kazemi2024data}
M.~Kazemi, R.~Majumdar, M.~Salamati, S.~Soudjani, and B.~Wooding, ``Data-driven abstraction-based control synthesis,'' \emph{Nonlinear Analysis: Hybrid Systems}, vol.~52, p. 101467, 2024.

\bibitem{Valiant1984}
\BIBentryALTinterwordspacing
L.~G. Valiant, ``A theory of the learnable,'' \emph{Commun. ACM}, vol.~27, no.~11, p. 1134–1142, nov 1984. [Online]. Available: \url{https://doi.org/10.1145/1968.1972}
\BIBentrySTDinterwordspacing

\bibitem{prajna2004safety}
S.~Prajna and A.~Jadbabaie, ``Safety verification of hybrid systems using barrier certificates,'' in \emph{International Workshop on Hybrid Systems: Computation and Control}.\hskip 1em plus 0.5em minus 0.4em\relax Springer, 2004, pp. 477--492.

\bibitem{yang2020efficient}
Z.~Yang, M.~Wu, and W.~Lin, ``An efficient framework for barrier certificate generation of uncertain nonlinear hybrid systems,'' \emph{Nonlinear Analysis: Hybrid Systems}, vol.~36, p. 100837, 2020.

\bibitem{majumdar2023neural}
R.~Majumdar, M.~Salamati, and S.~Soudjani, ``Neural abstraction-based controller synthesis and deployment,'' \emph{ACM Transactions on Embedded Computing Systems}, vol.~22, no.~5s, pp. 1--25, 2023.

\bibitem{haesaert2017data}
S.~Haesaert, P.~M. Van~den Hof, and A.~Abate, ``Data-driven and model-based verification via {B}ayesian identification and reachability analysis,'' \emph{Automatica}, vol.~79, pp. 115--126, 2017.

\bibitem{polgreen2017automated}
------, ``Automated experiment design for data-efficient verification of parametric {M}arkov decision processes,'' in \emph{International {C}onference on {Q}uantitative {E}valuation of {S}ystems}.\hskip 1em plus 0.5em minus 0.4em\relax Springer, 2017, pp. 259--274.

\bibitem{asalamati2020}
A.~Salamati, S.~Soudjani, and Zamani, ``Data-driven verification under signal temporal logic constraints,'' 2020.

\bibitem{salamati2024data}
A.~Salamati, A.~Lavaei, S.~Soudjani, and M.~Zamani, ``Data-driven verification and synthesis of stochastic systems via barrier certificates,'' \emph{Automatica}, vol. 159, p. 111323, 2024.

\bibitem{Angluin1987}
\BIBentryALTinterwordspacing
D.~Angluin, ``Learning regular sets from queries and counterexamples,'' \emph{Information and Computation}, vol.~75, no.~2, pp. 87--106, 1987. [Online]. Available: \url{https://www.sciencedirect.com/science/article/pii/0890540187900526}
\BIBentrySTDinterwordspacing

\bibitem{IHS2015}
\BIBentryALTinterwordspacing
M.~Isberner, F.~Howar, and B.~Steffen, ``The open-source learnlib - {A} framework for active automata learning,'' in \emph{Computer Aided Verification - 27th International Conference, {CAV} 2015, San Francisco, CA, USA, July 18-24, 2015, Proceedings, Part {I}}, ser. Lecture Notes in Computer Science, D.~Kroening and C.~S. Pasareanu, Eds., vol. 9206.\hskip 1em plus 0.5em minus 0.4em\relax Springer, 2015, pp. 487--495. [Online]. Available: \url{https://doi.org/10.1007/978-3-319-21690-4\_32}
\BIBentrySTDinterwordspacing

\bibitem{AP2018}
\BIBentryALTinterwordspacing
F.~Avellaneda and A.~Petrenko, ``Inferring dfa without negative examples,'' in \emph{International Conference on Graphics and Interaction}, 2018. [Online]. Available: \url{https://api.semanticscholar.org/CorpusID:61155441}
\BIBentrySTDinterwordspacing

\bibitem{GSPO2015}
R.~Groz, A.~Simao, A.~Petrenko, and C.~Oriat, ``Inferring finite state machines without reset using state identification sequences,'' in \emph{Testing Software and Systems}, K.~El-Fakih, G.~Barlas, and N.~Yevtushenko, Eds.\hskip 1em plus 0.5em minus 0.4em\relax Cham: Springer International Publishing, 2015, pp. 161--177.

\bibitem{kearns1994}
M.~J. Kearns and U.~Vazirani, \emph{An introduction to computational learning theory}.\hskip 1em plus 0.5em minus 0.4em\relax MIT press, 1994.

\bibitem{chen2016pac}
\BIBentryALTinterwordspacing
Y.-F. Chen, C.~Hsieh, O.~Leng\'{a}l, T.-J. Lii, M.-H. Tsai, B.-Y. Wang, and F.~Wang, ``Pac learning-based verification and model synthesis,'' in \emph{Proceedings of the 38th International Conference on Software Engineering}, ser. ICSE '16.\hskip 1em plus 0.5em minus 0.4em\relax New York, NY, USA: Association for Computing Machinery, 2016, p. 714–724. [Online]. Available: \url{https://doi.org/10.1145/2884781.2884860}
\BIBentrySTDinterwordspacing

\bibitem{prandini2023datadriven}
\BIBentryALTinterwordspacing
A.~Devonport and M.~Arcak, ``\BIBforeignlanguage{en}{Data-{Driven} {Estimation} of {Forward} {Reachable} {Sets}},'' in \emph{\BIBforeignlanguage{en}{Computation-{Aware} {Algorithmic} {Design} for {Cyber}-{Physical} {Systems}}}, M.~Prandini and R.~G. Sanfelice, Eds.\hskip 1em plus 0.5em minus 0.4em\relax Cham: Springer International Publishing, 2023, pp. 165--185, series Title: Systems \& Control: Foundations \& Applications. [Online]. Available: \url{https://link.springer.com/10.1007/978-3-031-43448-8_8}
\BIBentrySTDinterwordspacing

\bibitem{SHM2011}
\BIBentryALTinterwordspacing
B.~Steffen, F.~Howar, and M.~Merten, \emph{Introduction to Active Automata Learning from a Practical Perspective}.\hskip 1em plus 0.5em minus 0.4em\relax Berlin, Heidelberg: Springer Berlin Heidelberg, 2011, pp. 256--296. [Online]. Available: \url{https://doi.org/10.1007/978-3-642-21455-4_8}
\BIBentrySTDinterwordspacing

\bibitem{holzmann2018explicit}
\BIBentryALTinterwordspacing
G.~J. Holzmann, \emph{Explicit-State Model Checking}.\hskip 1em plus 0.5em minus 0.4em\relax Cham: Springer International Publishing, 2018, pp. 153--171. [Online]. Available: \url{https://doi.org/10.1007/978-3-319-10575-8_5}
\BIBentrySTDinterwordspacing

\bibitem{git-repo}
S.~Plambeck and A.~Salamati, ``{DNF} inference for safety analysis,'' https://github.com/TUHH-IES/golddnf4safety, 2025.

\end{thebibliography}
\end{document}